\documentclass[a4paper,12pt]{article} 
\title{Points of non-linearity of functions generated by random neural networks}

\usepackage{amsmath, amssymb, mathrsfs, amsthm, shorttoc, geometry, stmaryrd,wrapfig, epigraph}
\usepackage{mathtools, sidecap} 
\usepackage{hyperref}
\usepackage{xcolor}
\hypersetup{
colorlinks,
linkcolor={black},
citecolor={blue!50!black},
urlcolor={blue!80!black}
}

\usepackage[all]{xy}

\usepackage{tabularx}
\usepackage{longtable}
\numberwithin{equation}{subsection}

\usepackage{enumitem}

\usepackage{cleveref}
\crefformat{equation}{(#2#1#3)}

\AtBeginDocument{\renewcommand{\ref}[1]{\cref{#1}}}

\usepackage{pgf,tikz}
\usetikzlibrary{matrix, calc, arrows}

\usepackage{tikz-cd} 

\makeatletter
\newcommand*{\doublerightarrow}[2]{\mathrel{
  \settowidth{\@tempdima}{$\scriptstyle#1$}
  \settowidth{\@tempdimb}{$\scriptstyle#2$}
  \ifdim\@tempdimb>\@tempdima \@tempdima=\@tempdimb\fi
  \mathop{\vcenter{
    \offinterlineskip\ialign{\hbox to\dimexpr\@tempdima+1em{##}\cr
    \rightarrowfill\cr\noalign{\kern.5ex}
    \rightarrowfill\cr}}}\limits^{\!#1}_{\!#2}}}
\newcommand*{\triplerightarrow}[1]{\mathrel{
  \settowidth{\@tempdima}{$\scriptstyle#1$}
  \mathop{\vcenter{
    \offinterlineskip\ialign{\hbox to\dimexpr\@tempdima+1em{##}\cr
    \rightarrowfill\cr\noalign{\kern.5ex}
    \rightarrowfill\cr\noalign{\kern.5ex}
    \rightarrowfill\cr}}}\limits^{\!#1}}}
\makeatother

\newcommand{\on}[1]{\operatorname{#1}}
\newcommand{\bb}[1]{{\mathbb{#1}}}

\newcommand{\ca}[1]{{\mathcal{#1}}}
\newcommand{\bd}[1]{{\mathbf{#1}}}

\newcommand{\abs}[1]{\lvert#1\rvert}

\newcommand{\sub}{\subseteq}

\theoremstyle{definition}
\newtheorem{definition}{Definition}[section]

\theoremstyle{plain}

\newtheorem{proposition}[definition]{Proposition}
\newtheorem{lemma}[definition]{Lemma}
\newtheorem{theorem}[definition]{Theorem}

\theoremstyle{remark}

\usepackage{letltxmacro}
\LetLtxMacro{\phiorig}{\phi}
\renewcommand{\phi}{\varphi}

\author{David Holmes}
\date{\today}

\newcounter{nootje}
\newcommand{\beq}{\begin{equation}}
\newcommand{\eeq}{\end{equation}}
\newcommand{\beqs}{\begin{equation*}}
\newcommand{\eeqs}{\end{equation*}}

\renewcommand{\k}{k}

\tikzset{
  symbol/.style={
    draw=none,
    every to/.append style={
      edge node={node [sloped, allow upside down, auto=false]{$#1$}}}
  }
}

\begin{document}
\maketitle
\begin{abstract} 
We consider functions from $\bb R \to \bb R$ output by a neural network with 1 hidden activation layer, arbitrary width, and ReLU activation function. We assume that the parameters of the neural network are chosen uniformly at random with respect to various probability distributions, and compute the expected distribution of the points of non-linearity. We use these results to explain why the network may be biased towards outputting functions with simpler geometry, and why certain functions with low information-theoretic complexity are nonetheless hard for a neural network to approximate. 
\end{abstract}




\newcommand{\Mtildes}{ \widetilde{\ca M}^\Sigma}
\newcommand{\sch}[1]{\textcolor{blue}{#1}}

\newcommand{\Mbar}{\overline{\ca M}}
\newcommand{\MD}{\ca M^\blacklozenge}
\newcommand{\Md}{\ca M^\lozenge}
\newcommand{\DRL}{\mathsf{DRL}}
\newcommand{\DR}{\mathsf{DR}}
\newcommand{\DRC}{\mathsf{DRC}}
\newcommand{\isom}{\stackrel{\sim}{\longrightarrow}}
\newcommand{\Ann}[1]{\on{Ann}(#1)}
\newcommand{\fm}{\mathfrak m}
\newcommand{\Mdk}{\Mbar^{\m, 1/\k}}
\newcommand{\field}{k}
\newcommand{\Mdm}{\Mbar^\m}
\newcommand{\m}{{\bd m}}
\newcommand{\cat}[1]{\bd{#1}}
\newcommand{\M}{\mathsf{M}}
\newcommand{\ghost}[1][M]{{\bar{\mathsf{#1}}}}
\newcommand{\gp}{\mathsf{gp}}
\newcommand{\fib}{\mathsf{cat}}
\newcommand{\et}{\mathsf{\acute{e}t}}
\renewcommand{\sf}[1]{\mathsf{#1}}
\newcommand{\Pic}{\mathfrak{Pic}}
\newcommand{\Sym}{\on{Sym}}
\newcommand{\Chow}{\on{CH}}
\newcommand{\Spec}{\on{Spec}}
\newcommand{\Picabs}{\mathfrak{Pic}}
\newcommand{\Picrel}{\mathfrak{Pic}^{\mathrm{rel}}}
\newcommand{\CHop}{\Chow}
\newcommand{\divCHop}{\on{divCH}}
\newcommand{\DRop}{\mathsf{DR}} 
\newcommand{\LogChow}{\on{LogCH}}
\newcommand{\divLogChow}{\on{divLogCH}}
\newcommand{\LogDR}{\sf{LogDR}}
\newcommand{\Pictdz}{\mathfrak{Jac}}
\renewcommand{\log}{\sf{log}}
\newcommand{\trop}{\sf{trop}}
\newcommand{\op}{\sf{op}}
\newcommand{\aj}{\sf{aj}}
\newcommand{\GL}{\on{GL}}
\newcommand{\J}{\sf{J}}
\newcommand{\PL}{\sf{PL}}

%

\section{Introduction}

It has been suggested \cite{Valle-Perez2018Deep-learning-g,Mingard2019Neural-networks,Mingard2020Is-SGD-a-Bayesi} that neural networks are biased in favour of outputting `simple' functions. The above papers interpret simplicity in an information-theoretic fashion, suggesting that functions output by neural networks tend to have small information-theoretic complexity. The goal of this paper is to illustrate that, at least in some contexts, more `geometric' notions of simplicity may capture this bias more accurately. 

One example of this may be seen by considering a simple periodic function such as the triangular sawtooth $x \mapsto \abs{x - \lfloor x \rfloor - \frac{1}{2}}$. This function has low information-theoretic complexity, but is relatively hard \cite{ziyin2020neural} for a neural network with ReLU (or other non-periodic) activation to learn. 

We will show that the points of non-linearity of a function output by a neural network tend to be either few in number, or clustered together in one place, depending on the setup. Either way, this illustrates why an information-theoretically simple function (such as a periodic function) can nonetheless be hard for a neural network to approximate. 

\subsection{Neural networks with random weights}

According to the heuristics of \cite{Mingard2020Is-SGD-a-Bayesi}, training a neural network by stochastic gradient descent may be well-approximated by assigning the weights and biases at random, \emph{conditional} on a good fit with the training data. For the sake of simplicity, in this preliminary work we do not use training data; we work simply with random neural networks, obtained by assigning each neuron a weight and bias chosen at random. 

Again for simplicity, we consider a restricted class of neural networks, with a single (ReLU) activation layer of width $w$, and exactly one input and one output neuron. A choice of weights and biases for the neurons thus produces a piecewise linear function from $\bb R \to \bb R$. 

A function produced in this way has a finite set of points of non-linearity; in fact there are at most $w$ such points.  Their distribution will of course depend on the distribution from which the weights and biases of the neurons are selected; below we compute precisely the distribution of the points of non-linearity for three different choices of distribution onto weights and biases of the neurons. 

Perhaps surprisingly, we will see that the distribution of the points of non-linearity depends very heavily on the distribution of the parameters. If the bias towards simple functions underlies generalisation properties of over-parameterised neural networks (as proposed in \cite{Mingard2020Is-SGD-a-Bayesi}), this may help to explain why some gradient descent schemes generalise better than others (as they approximate random sampling with respect to distributions more heavily favouring simple functions).

\subsection{Main results}
Our neural network has parameters taking values in some measurable subset of $\Theta$ of a real vector space. The probabilities of seeing a given number of points of non-linearity turn out to be highly dependent on the \emph{shape} of the parameter space $\Theta$, but independent of the \emph{size} of $\Theta$. We will consider three different `shapes': one `rectangular' (\ref{thm:rectangular}), one `Gaussian' (\ref{thm:gaussian}), and one `spherical'(\ref{theorem:spherical}). 

We consider neural networks with one hidden activation layer of with $w$ (a positive integer); see \ref{sec:param_space_notation} for a formal description of this setup. We fix $R \in (0, \infty]$, and view the resulting real-valued function as being defined on the interval $(-R, R)$. 

\subsubsection{Rectangular parameter space}
Here we fix a positive real number $T$, and then choose both the weight and the bias for each of the $w$ neurons independently and uniformly at random from the interval $(-T,T)$. In other words, the vector of biases is chosen uniformly at random from a box $[-T,T]^w \sub \bb R^w$, and the same holds for the vector of weights. 

A PL function generated in this way has at most $w$ points of non-linearity (\ref{lem:max_w}). In fact, for $R$ finite, the number of points of non-linearity follows a binomial distribution:
\begin{theorem}[{\ref{prop:rectangular_expectations}}] \label{thm:rectangular}
Suppose $R \in (0,\infty)\subseteq \bb R$. Given any $w' \in \{1, 2, \dots, w\}$, the probability of a function generated by this neural network having exactly $w'$ points of non-linearity is 
\begin{equation*}
\binom{w}{w'}\ca P^{w'}(1- \ca P)^{w - w'}
\end{equation*}
where $\binom{w}{w'}$ is the binomial coefficient, and
\begin{equation}
\ca P = \begin{cases}
\frac{R}{2} & \text{if } 0 < R \le 1\\
1-\frac{1}{2R} & \text{if } R \ge 1. \\
\end{cases}
\end{equation}
In particular, the expected number of points of non-linearity is given by 
\begin{equation*}
\bb E(\#D_\theta) = w \ca P < w. 
\end{equation*}
\end{theorem}

\paragraph{Functions on an unbounded domain}\label{sec:rec_unbounded}

Suppose that we use the same parameter space $\Theta$, but we now take $R = \infty$; in other words, we view our PL functions as having domain $\bb R$. Then for almost all $\theta \in \Theta$, the PL function will have $w$ points of non-linearity. However, the distribution of these points is far from uniform. Differentiating the above result, we find that the probability density function of the distribution of these $w$ points is given by (see \ref{fig:comparison})
\begin{equation}
\bb P(\abs{x} = r) = \begin{cases}
\frac{w}{2} & \text{if } 0 \le r \le 1\\
\frac{w}{2r^2} & \text{if } 1 \le r. \\
\end{cases}
\end{equation}


\subsubsection{Gaussian parameter space}
We adopt the same notation as in the rectangular case, but instead of choosing weights and biases uniformly at random from an interval $[-T,T]$, we now fix a positive real number $\nu$ and choose both the weight and bias for each neuron at random from a normal distribution with mean 0 and variance $\nu$. 

In other words, the vector of biases (the 'affine part' at each neuron) is chosen at random from a product of normal distributions on $\bb R^w$, and the same holds for the vector of weights (the 'linear part' at each neuron). 

Almost exactly the same formulae hold as above, except the expression for $\ca P$ is different:
\begin{theorem}[{\ref{gaussian_P}}]\label{thm:gaussian}
Suppose $R$ is finite. Given any $w' \in \{1, 2, \dots, w\}$, the probability of a function generated by this neural network having exactly $w'$ points of non-linearity is 
\begin{equation*}
\binom{w}{w'}\ca P^{w'}(1- \ca P)^{w - w'}
\end{equation*}
where $\binom{w}{w'}$ is the binomial coefficient, and
\begin{equation}
\ca P = \frac{2}{\pi}\arctan R. 
\end{equation}
In particular, the expected number of points of non-linearity is given by 
\begin{equation*}
\bb E(\#D_\theta) = w \ca P < w. 
\end{equation*}
\end{theorem}

\paragraph{Unbounded domain}
Again, in the case $R = \infty$, differentiating the above result shows that the distribution of the $w$ points of non-linearity is given by (see \ref{fig:comparison})
\begin{equation}
\bb P(\abs{x} = r) = \frac{2w}{\pi (1 + r^2)}. 
\end{equation}


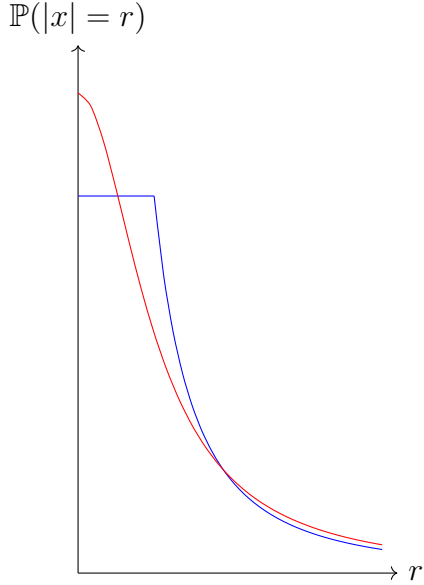
\begin{SCfigure}
\begin{tikzpicture}[scale = 2]
  \draw[->] (0, 0) -- (2.1, 0) node[right] {$r$};
  \draw[->] (0, 0) -- (0, 3.5) node[above] {$\bb P(\abs{x} = r)$};
  \draw[scale=0.5, domain=0:1, smooth, variable=\x, blue] plot ({\x}, {5});
      \draw[scale=0.5, domain=1:4, smooth, variable=\x, blue] plot ({\x}, {10/(2* \x * \x) });
    \draw[scale=0.5, domain=0:4, smooth, variable=\x, red] plot ({\x}, {6.36619/(1 + \x * \x) });
\end{tikzpicture}
\caption{\label{fig:comparison} Distribution of points of non-linearity for $w=10$ and rectangular parameter distribution (blue) and Gaussian parameter distribution (red). }
\end{SCfigure}

\subsubsection{Spherical parameter space}

We again fix a positive real number $T$, and choose the \emph{biases} of the neurons uniformly and independently at random from an interval $[-T,T]$. However, the weights (i.e. the linear part of the affine linear transformation at each neuron) are chosen uniformly at random from a sphere of radius $T$ in $\bb R^w$. 

In other words, the vector of biases is chosen at random from a product of normal distributions on $\bb R^w$, and the vector of weights is chosen uniformly at random in 
\begin{equation*}
\{ L \in \bb R^w : \abs{L} \le T\} \sub \bb R^w. 
\end{equation*}



In this context it does not seem so easy to compute the exact probability of a given number of points of non-linearity occurring. However, at least for $R \le 1$, we can compute the expected number of points of non-linearity. 
\begin{theorem}[{\ref{prop:main_spherical}}]\label{theorem:spherical}
Assume $0 < R \le 1$. 
For $w$ even we have
\begin{equation}
\bb E(\#D_\theta) = \frac{Rw2^{w}}{(w+1)\pi} \binom{w-1}{w/2}^{-1} \sim R\sqrt{2w/\pi}. 
\end{equation}
For odd $w$ we have
\begin{equation}
\bb E(\#D_\theta) = \frac{Rw^2}{2^{w-1}(w+1)} \binom{w-1}{(w-1)/2} \sim R\sqrt{2w/\pi}. 
\end{equation}
Here $\sim$ means that the ratio tends to $1$ as $w$ tends to infinity. \end{theorem}
Since $\sqrt{w}$ is much smaller than $w$, this indicates that such functions tend strongly to having few points of linearity. 


\subsection{Possible generalisations and extensions}

\subsubsection{Training data}
If the training data forces $w$ points of non-linearity, then of course they will occur with probability $1$ among parameters fitting the training data. However, the fact that we generally work with highly over-parametrised models says that this will in general not be the setup. We expect that similar results will hold (and be provable with similar techniques) in the presence of training data, as long as certain `over-parametrisation' conditions are satisfied. 

\subsubsection{Higher dimensions}

Instead of looking at a single input neuron (i.e. real valued functions), we can consider any finite number $i$ of input neurons, leading to functions $\bb R^i \to \bb R$. Here the locus of points where the function is not linear will not be finite; as such, instead of using the cardinality of that set as a measure of simplicity, we will instead use its Hausdorff measure. We believe that similar results can be shown, by similar methods. 

\subsubsection{Different activation functions}
Our results are valid not just for the ReLU activation function, but in fact any piecewise-linear activation function which has a unique point of non-linearity at the origin. Generalising to other PL activation functions will require no major changes. For differentiable activation functions which are asymptotically linear, we can replace the measure of the locus of points of non-linearity by (for example) the integral of the square of the largest eigenvalue of the Hessian. Again, we expect similar results, but different techniques will be required to prove them. 

\subsection*{Acknowledgements}
The author is grateful to Joar Skalse for helpful comments.

\section{Parameter space notation}\label{sec:param_space_notation}

The \emph{activation} is a function $\phi\colon \bb R \to \bb R$ which is linear exactly away from $0$ (for example, the ReLU activation $x \mapsto \max(x,0)$). 

We consider neural networks with one hidden activation layer, defining functions from $(-R,R)\subseteq \bb R$ to $\bb R$. We write $w$ for the width (a positive integer). Our parameter space $\Theta$ for the neural network is naturally a product of $4$ pieces:
\begin{enumerate}
\item
 a linear part in the first layer; this gives a subspace of $\bb R^w$, denoted $\Theta_{L}$
 \item a translation part in the first layer; this again gives a subspace of $\bb R^w$, denoted $\Theta_{T}$
\item
 a linear part in the final layer; this again gives a subspace of $\bb R^w$, denoted $\Theta'_{L}$
 \item a translation part in the final layer; this gives a subspace of $\bb R$, denoted $\Theta'_{T}$. 
\end{enumerate}
So $\Theta = \Theta_L \times \Theta_T \times \Theta'_L \times \Theta'_T$. A point $\theta = (\theta_L, \theta_T, \theta'_L, \theta'_T) \in \Theta$ determines a function 
\begin{equation}
f_\theta\colon \bb R \to \bb R 
\end{equation}
by the formula
\begin{equation}
x \mapsto \theta_T' + \theta_L' \cdot \phi^w(\theta_T + x \cdot \theta_L), 
\end{equation}
where $\phi^w\colon \bb R^w \to \bb R^w$ applies the activation function $\phi$ to each coordinate.

Now $\Theta'_{L}$ and $\Theta'_{T}$ have no effect on the number of points of non-linearity (outside some measure-zero subset of $\Theta_{L}'$, which we ignore). So it suffices to describe the spaces $\Theta_L$ and $\Theta_T$, and their probability measures. 

\section{Rectangular weights and biases}

We fix a positive real number $T$. We define 
\begin{equation}
\Theta_L = \Theta_T = [-T,T]^w, 
\end{equation}
a box of side-length $T$ and dimension $w$, centred at $0 \in \bb R^w$. We equip it with the lebesgue measure. In other words, 
\begin{itemize}
\item
the `translation' part of the first layer is chosen uniformly at random between $-T$ and $T$ at each neuron; 
\item the scaling factor at each neuron in the first layer is chosen uniformly at random between $-T$ and $T$. 
\end{itemize}


%

Given $\theta \in \Theta$, we write $D_\theta \subseteq (-R, R)$ for the set of points of non-linearity of the function given by the parameters $\theta$. 

\begin{lemma}\label{lem:max_w}
$\# D_\theta \le w$, and this maximum can be achieved. 
\end{lemma}
\begin{proof}
Suppose the image of $X$ is not contained in a coordinate hyperplane of $\bb R^w$. Then $D_\theta$ is exactly the image under a linear map of the intersection of the image of $X$ with the coordinate hyperplanes in $\bb R^w$, of which there are at most $w$. 

On the other hand, if the image of $X$ is contained in the intersection of exactly $w'$ of the coordinate hyperplanes, then the image of $X$ hits at most $w-w'$ other coordinate hyperplanes. 
\end{proof}

Fix $w' \in \{0, 1, \dots, w\}$. 
We compute $\bb P(\# D_\theta = w')$. A point in $c \in \bb R^w$ is chosen uniformly at random in a box around $0$ of side-length $2T$. Another point $l \in \bb R^w$ is chosen uniformly at random in a box around $0$ of side-length $2RT$. Then we consider the line segment in $\bb R^w$ joining $c - l$ and $c + l$, and we want to compute the probability of this segment meeting any of the coordinate hyperplanes; given $1 \le i \le w$ write $\ca P_i$ for the probability of our line segment meeting the $i$th coordinate hyperplanes; this is independent of $i$, so we also write it $\ca P$. 

\begin{lemma}If $R \ge 2$ then 
$\ca P = 1-\frac{1}{2R}$. 
If $0 < R \le 1$ then $\ca P = \frac{r}{R}$. 
\end{lemma}
\begin{proof}
Without loss of generality, $i=1$. Then for fixed $c$ the probability of intersecting the point $x_1 = 0$ is
\begin{equation*}
\max(1 - \frac{\abs{c_1}}{RT}, 0). 
\end{equation*}
Integrating over $\abs{c_1}$ from $0$ to $T$ yields the result. 
\end{proof}

Since the probabilities of hitting the various axes are independent, we deduce
\begin{proposition}\label{prop:rectangular_expectations}
For $w' \in \{0, 1, \dots, w\}$ the probability of a function generated by this neural network having exactly $w'$ points of non-linearity is 
\begin{equation}
\bb P(\# D_\theta = w') = \binom{w}{w'}\ca P^{w'}(1- \ca P)^{w - w'}. 
\end{equation}
The expected number of points of non-linearity is given by 
\begin{equation}
\bb E(\# D_\theta) = \sum_{w' \in \{0, \dots, w\}}\frac{\mu(\Theta_{w'})}{\mu(\Theta)} w'  = w\ca P. 
\end{equation}
\end{proposition}

\section{Gaussian weights and biases}

The Gaussian version is very similar, except that instead of the width $T$ of the interval, we work with the variance $\nu$ of the normal distribution. We have $\Theta_L = \Theta_T = \bb R^w$, each of which is equipped with a product of normal distributions with variance $\nu$. 

As before we write $D_\theta \subseteq (-R, R)$ for the set of points of non-linearity for the function produced by some $\theta \in \Theta$. Given $w' \in \{0, 1, \dots, w\}$, we begin by computing $\bb P(\# D_\theta = w')$. 

A point in $c \in \bb R^w$ has coordinates chosen independently from a normal distribution with variance $\nu$. Another point $l \in \bb R^w$ has coordinates chosen independently from a normal distribution with variance $R^2\nu$. Then we consider the line segment in $\bb R^w$ joining $c - l$ and $c + l$, and we want to compute the probability of this segment meeting any of the coordinate hyperplanes; given $1 \le i \le w$ write $\ca P_i$ for the probability of our line segment meeting the $i$th coordinate hyperplanes; this is independent of $i$, so we also write it $\ca P$. 

\begin{lemma}\label{gaussian_P}
\begin{equation}
\ca P = \bb P(\abs{\ca N(0, \sigma^2R^2)} \ge \abs{\ca N(0, \sigma^2)}) = \frac{2}{\pi} \arctan R. 
\end{equation}
\end{lemma}
\begin{proof}
The first equality is the definition. For the second, writing $\phi = \frac{e^{-z^2/2}}{\sqrt{2 \pi}}$ for the probability density function of $\ca N(0,1)$, the rotational symmetry of $\phi(x)\phi(y)$ yields
\begin{equation}
\begin{split}
\bb P(\abs{\ca N(0, R^2)} \ge \abs{\ca N(0, 1)}) & = \bb P(\abs{\ca N(0, R^2)/\ca N(0, 1)}  \ge 1 ) \\
& = \bb P(\abs{\ca N(0, 1)/\ca N(0, 1)}  \ge 1/R ) \\
& = \frac{2}{\pi}\arctan R. \qedhere
\end{split}
\end{equation}
\end{proof}


Since the probabilities of hitting the various axes are again independent, the exact same formulae as in \ref{prop:rectangular_expectations} hold; for $w' \in \{0, 1, \dots, w\}$ the probability of a function generated by this neural network having exactly $w'$ points of non-linearity is 
\begin{equation}
\bb P(\# D_\theta = w') = \binom{w}{w'}\ca P^{w'}(1- \ca P)^{w - w'}, 
\end{equation}
and the expected number of points of non-linearity is given by 
\begin{equation}
\bb E(\# D_\theta) = \sum_{w' \in \{0, \dots, w\}}\frac{\mu(\Theta_{w'})}{\mu(\Theta)} w'  = w\ca P. 
\end{equation}

\section{Spherical weights, uniform biases}

We again fix a positive real parameter $T$. We define $\Theta_T = [-T, T]^w \subseteq \bb R^w$, and 
\begin{equation}
\Theta_L = \{ L \in \bb R^w : \abs{L} \le T\} \sub \bb R^w, 
\end{equation}
where $\abs{L}$ is the Euclidean norm. Just as in the previous cases, we have
\begin{lemma}
For any $\theta \in \Theta$, we have $\#D_\theta \le w$, and this maximum can be achieved. 
\end{lemma}
%


 To compute the probabilities precisely (as in \ref{prop:rectangular_expectations,gaussian_P}) seems difficult, but by a simple application of classical results from geometric probability we will be able to compute the expected number of points of nonlinearity, on small domains. More precisely, we fix a real number $R \in (0,1]$, and for $\theta \in \Theta$ we define $D_\theta$ to be the set of points of non-linearity of the resulting function $f_\theta\colon (-R,R) \to \bb R$. The expected number of points of non-linearity is given by
\begin{equation}
\bb E(\#D_\theta) = \sum_{w' \in \{0, \dots, w\}}\bb P(\#D_\theta = w') w'. 
\end{equation}
For example, if $\bb E(\#D_\theta) \approx w$ this would tell us that most choices of parameters yield $\#D_\theta = w$. On the other hand, if $\bb E(\#D_\theta) \approx 0$ this would tell us that most choice of parameter give an affine-linear function.



\begin{proposition}\label{prop:main_spherical}
Assume $0 < R \le 1$. 
For $w$ even we have
\begin{equation}
\bb E(\#D_\theta) = \frac{Rw2^{w}}{(w+1)\pi} \binom{w-1}{w/2}^{-1} \sim R\sqrt{2w/\pi}. 
\end{equation}
For odd $w$ we have
\begin{equation}
\bb E(\#D_\theta) = \frac{Rw^2}{2^{w-1}(w+1)} \binom{w-1}{(w-1)/2} \sim R\sqrt{2w/\pi}. 
\end{equation}
Here $\sim$ means that the ratio tends to $1$ as $w$ tends to infinity. 
\end{proposition}

\begin{proof}
As before, only the first component 
\begin{equation}
f_1\colon \bb R \to \bb R^w; x \mapsto \theta_T + x\cdot \theta_L
\end{equation}
of the map $f_\theta$ has any impact on the number $\#D_\theta$; more precisely, $\#D_\theta$ is the number of intersection points of the image of $[-R,R]$ under $f_1$ with the coordinate hyperplanes in $\bb R^w$ (excluding the measure-zero case where the image crosses the intersection of two or more coordinate hyperplanes). 

We now relate the problem to a variation on Buffon's needle. The image of $[-R,R]$ is a line segment in $\bb R^w$, with centre a point in $[-T, T]^w$ chosen uniformly at random, and endpoint chosen uniformly at random in a sphere of radius $RT$ around that centre. We want to compute the expected number of intersection points with the coordinate hyperplanes. 

For now we fix the length $2s \in [0,2RT]$ of the needle, and compute the expectation; later we will integrate over $s$. By additivity of expectations, we are reduced to computing the expected number $\frac{1}{w}\bb E(\#D_w)$ of intersection points with a single coordinate hyperplane. 
By symmetry, the expected number of intersection points with a coordinate hyperplane is the same as the expected number of intersection points of a needle of length $2s$, dropped uniformly at random in the plane, with the subset 
\begin{equation}
\{x \in \bb R^w : x_1 \in 2T\bb Z\}. 
\end{equation}
By \cite[page 130]{Klain1997Introduction-to} this expectation is given by\footnote{We write $\Omega$ where Klain and Rota write $\omega$, to make the distinction from the width $w$ clearer. }
\begin{equation}
\bb E = \frac{\Omega_1\Omega_{w-1}}{w\Omega_w}\frac{s}{T}
\end{equation}
where for a non-negative integer $k$ we have
\begin{equation}
\Omega_{2k} = \frac{\pi^k}{k!}
\end{equation}
and
\begin{equation}
\Omega_{2k+1} = \frac{2^{2k + 1}\pi^k k!}{(2k+1)!}. 
\end{equation}
We find for even $w$ that 
\begin{equation}
\bb E = \frac{2^w s}{w \pi T}\binom{w-1}{w/2}^{-1}
\end{equation}
and for odd $w$ that 
\begin{equation}
\bb E = \frac{s}{2^{w-1} T}\binom{w-1}{(w-1)/2}. 
\end{equation}
To simplify subsequent computations, we write $\bb E' = T\bb E/s$, which depends only on $w$. 
To complete the computation of the expectations we must integrate over $s \in [0,RT]$. However, we do not integrate with respect to the uniform distribution on $[0,RT]$; rather we want the endpoint of our needle to be chosen uniformly in a sphere. As such, the expectation for hitting one hyperplane is 
\begin{equation}
\frac{1}{w}\bb E(\#D_w) = B(w, TR)^{-1} \int_{s=0}^{RT} \frac{s}{T} \bb E' S(w, s) ds
\end{equation}
where $$B(w, TR) = \frac{\pi^{w/2}}{\Gamma(\frac{w}{2} + 1)} (TR)^w$$ is the volume of a ball of radius $TR$ and dimension $w$, and 
$$S(w, s) = \frac{2\pi^{w/2}}{\Gamma(\frac{w}{2})} s^{w-1}$$ 
is the surface area of a ball of dimension $w$ and radius $s$. This turns into
\begin{equation}
\begin{split}
\frac{1}{w}\bb E(\#D_w) & =  \frac{\Gamma(\frac{w}{2} + 1)}{\pi^{w/2}(TR)^{w}} \int_{s=0}^{RT} \frac{s}{T} \bb E'  \frac{2\pi^{w/2}}{\Gamma(\frac{w}{2})} s^{w-1} ds\\
& = R \frac{w}{w+1} \bb E'. 
\end{split}
\end{equation}
For the asymptotics we apply the central binomial coefficient formula 
\begin{equation}
\binom{2k}{k} \sim \frac{4^k}{\sqrt{k\pi}}, 
\end{equation}
and for even $w$ we also use 
\begin{equation}
\binom{2k-1}{k-1} = \frac{1}{2}\binom{2k}{k}. 
\end{equation}
\end{proof}

\bibliographystyle{alpha} 
\bibliography{../../prebib.bib}

\newcommand{\etalchar}[1]{$^{#1}$}
\def\cprime{$'$}
\begin{thebibliography}{MSVP{\etalchar{+}}19}

\bibitem[KR97]{Klain1997Introduction-to}
Daniel~A Klain and Gian-Carlo Rota.
\newblock {\em Introduction to geometric probability}.
\newblock Cambridge University Press, 1997.

\bibitem[MSVP{\etalchar{+}}19]{Mingard2019Neural-networks}
Chris Mingard, Joar Skalse, Guillermo Valle-P{\'e}rez, David
  Mart{\'\i}nez-Rubio, Vladimir Mikulik, and Ard~A Louis.
\newblock Neural networks are a priori biased towards boolean functions with
  low entropy.
\newblock {\em arXiv preprint arXiv:1909.11522}, 2019.

\bibitem[MVPSL21]{Mingard2020Is-SGD-a-Bayesi}
Chris Mingard, Guillermo Valle-P{\'e}rez, Joar Skalse, and Ard~A Louis.
\newblock Is {SGD} a {B}ayesian sampler? {W}ell, almost.
\newblock {\em The Journal of Machine Learning Research}, 22(1):3579--3642,
  2021.

\bibitem[VPCL18]{Valle-Perez2018Deep-learning-g}
Guillermo Valle-Perez, Chico~Q Camargo, and Ard~A Louis.
\newblock Deep learning generalizes because the parameter-function map is
  biased towards simple functions.
\newblock {\em arXiv preprint arXiv:1805.08522}, 2018.

\bibitem[ZHU20]{ziyin2020neural}
Liu Ziyin, Tilman Hartwig, and Masahito Ueda.
\newblock Neural networks fail to learn periodic functions and how to fix it.
\newblock {\em Advances in Neural Information Processing Systems},
  33:1583--1594, 2020.

\end{thebibliography}

\vspace{+16 pt}

\noindent David~Holmes\\
\textsc{Mathematisch Instituut, Universiteit Leiden, Postbus 9512, 2300 RA Leiden, Netherlands} \\
  \textit{E-mail address}: \texttt{holmesdst@math.leidenuniv.nl}

\end{document}